\documentclass[hidelinks,letterpaper,10pt,conference]{ieeeconf}

\newif\ifanonymous
\anonymousfalse         % RA-L submission: hide authors, affiliations, funding
\IEEEoverridecommandlockouts
\usepackage[dvipsnames]{xcolor}
\usepackage{cite}
\usepackage{amsmath,amssymb,amsfonts}
\usepackage{graphicx}
\usepackage{epstopdf}
\usepackage{textcomp}
\usepackage{comment}
\usepackage{subcaption}
\usepackage[font=small]{caption}
\usepackage{float}
\usepackage{bm}
\usepackage{booktabs}
\usepackage{soul}
\usepackage{algpseudocode}
\usepackage{tikz}
\usetikzlibrary{arrows.meta,calc,positioning,fit}
\let\labelindent\relax
\usepackage{enumitem}

\usepackage{hyperref}

\setlist[enumerate]{leftmargin=*}
\setlist[itemize]{leftmargin=*}

\usepackage{amsthm}

\newtheorem{lemma}{Lemma}
\newtheorem{corollary}{Corollary}

\newtheorem{theorem}{Theorem}
\newtheorem{remark}{Remark}

\def\BibTeX{{\rm B\kern-.05em{\sc i\kern-.025em b}\kern-.08em
    T\kern-.1667em\lower.7ex\hbox{E}\kern-.125emX}}

\title{\Huge Exact Thrust-Reversal Limits of Bidirectional Propellers under Bounded Motor Inputs}

\ifanonymous
\author{Anonymous Authors}
\else
\author{Ahmed Ali$^{*}$, Chiara Gabellieri$^{*}$, Antonio Franchi$^{*,\dagger}$
\thanks{$^*$ Robotics and Mechatronics Department, Electrical Engineering, Mathematics, and Computer Science (EEMCS) Faculty, University of Twente, 7500 AE Enschede, The Netherlands. {\footnotesize ahmed.ali@utwente.nl, c.gabellieri@utwente.nl, schol@r-franchi.eu}}
\thanks{$^\dagger$ Department of Computer, Control and Management Engineering, Sapienza University of Rome, 00185 Rome, Italy. {\footnotesize schol@r-franchi.eu}}
\thanks{This work was partially funded by the Horizon Europe research agreement no. 101120732 (AUTOASSESS) and NWO OTP AVIATOR.}
}
\fi

\begin{document}

\maketitle

\begin{abstract}
Bidirectional propellers are often treated as signed thrust sources, but their thrust is a signed-quadratic function of rotor speed. 
Thus, thrust reversal necessarily occurs through zero rotor speed, where the ability of a bounded motor torque to change thrust collapses. 
This work formalizes this obstruction by studying exact thrust-trajectory reproducibility under bounded motor inputs with prescribed smoothness. 
We derive a normalized thrust-coordinate model with vanishing input gain at zero thrust, and prove necessary and sufficient reproducibility conditions in terms of the zero-crossing order of the desired thrust. 
Generic reversals, in which thrust crosses zero with nonzero slope, require unbounded motor input; the resulting conditions provide direct design rules for shaping thrust reversals that avoid singular motor commands. 
We also derive the corresponding current and voltage regularity requirements for a DC motor driving a bidirectional propeller. 
Experiments on a motor-propeller setup validate the predicted reversal-order effects, showing localized current/voltage peaks and thrust-tracking degradation for linear reversals, but not for higher-order reversals. 
These results expose an intrinsic actuator-level limitation that must be considered in force, acceleration, and interaction-control references for aerial robots.
\end{abstract}

%%%%%%%%%%%%%%%%%%%%%%%%%%%%%%%%%%%%%%%%
\section{Introduction}
%%%%%%%%%%%%%%%%%%%%%%%%%%%%%%%%%%%%%%%%

Robots that physically interact with their environment must regulate motion and force through the same mechanical interface. 
Impedance control prescribed the dynamic relation between motion and interaction forces~\cite{hogan1985impedance1}, hybrid position/force control separated task directions for motion and force objectives~\cite{raibert1981hybrid}, and operational-space control formulated motion and force directly at the end effector~\cite{khatib1987operational}. 
These ideas are central in aerial robotics, where contact tasks are executed by floating-base systems with tightly coupled actuation, sensing, and interaction dynamics.

Aerial physical interaction has motivated platforms and controllers beyond standard underactuated multirotors. 
Surveys of aerial manipulation describe systems that grasp, push, inspect, and physically interact with the environment, and classify their platforms, interaction mechanisms, modeling, estimation, and control methods~\cite{ruggiero2018aerial,khamseh2018aerial}. 
The AEROARMS project developed aerial manipulation technologies for outdoor inspection and maintenance with dual arms and multidirectional thrusters~\cite{ollero2018aeroarms}. 
The flying-end-effector paradigm showed full-pose control and full-wrench interaction through a rigidly attached end effector~\cite{ryll2019flying}. 
Bodie et al. demonstrated omnidirectional aerial platforms for contact-based inspection and later validated active interaction force control in different flight orientations~\cite{bodie2019omnidirectional,bodie2021active}.

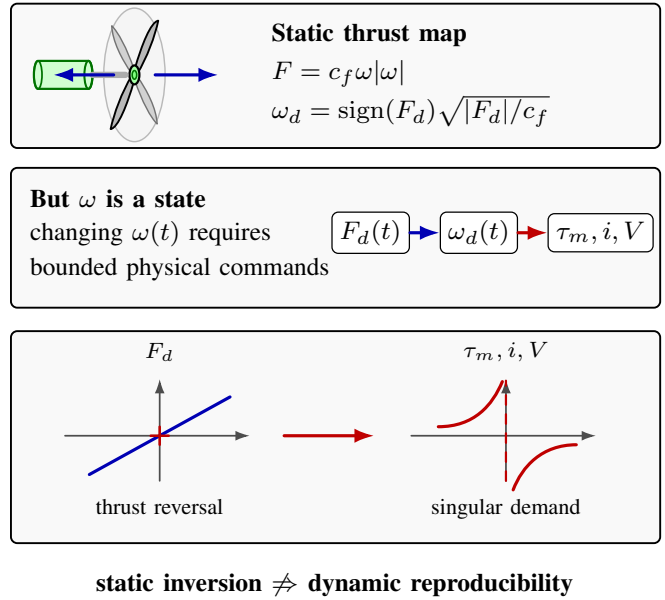
\begin{figure}[t]
    \centering
    \resizebox{\columnwidth}{!}{%
    \begin{tikzpicture}[
        line cap=round,
        line join=round,
        >=Latex,
        font=\footnotesize,
        box/.style={draw=black, rounded corners=2pt, line width=0.6pt, fill=gray!5, inner sep=4pt},
        bluearrow/.style={blue!70!black, line width=1.0pt, -{Latex[length=2.2mm,width=1.6mm]}},
        redarrow/.style={red!75!black, line width=1.0pt, -{Latex[length=2.2mm,width=1.6mm]}},
        axis/.style={black!70, line width=0.55pt, -{Latex[length=1.6mm,width=1.1mm]}},
        curve/.style={blue!70!black, line width=1.0pt},
        redcurve/.style={red!75!black, line width=1.0pt},
        motor/.style={draw=green!45!black, line width=0.7pt, fill=green!20},
        blade/.style={draw=black, line width=0.6pt, fill=gray!55}
    ]

    % Exact bounding box to reduce external whitespace
    \path[use as bounding box] (-0.02,0.05) rectangle (7.22,-6.78);

    % =====================================================
    % Panel 1: static map
    % =====================================================
    \node[box, minimum width=7.2cm, minimum height=1.60cm, anchor=north west] at (0,0) {};

    % Horizontal shift of propeller icon and bidirectional thrust arrows
	\def\propellerXshift{0.23cm}

    \begin{scope}[xshift=\propellerXshift]

    % Propeller icon, drawn as an oblique disk viewed from an angle
    \coordinate (hub) at (1.15,-0.80);

    % -----------------------------------------------------
    % Motor as a small oblique cylinder
    % -----------------------------------------------------
    \coordinate (motorC) at (0.36,-0.80);
    \def\motorHalfLen{0.28}
    \def\motorRx{0.045}
    \def\motorRy{0.155}

    % Motor body, elliptic front/back caps
    \draw[draw=green!45!black, line width=0.7pt, fill=green!18]
        ($(motorC)+(-\motorHalfLen,-\motorRy)$)
        -- ($(motorC)+(\motorHalfLen,-\motorRy)$)
        arc[start angle=-90,end angle=90,x radius=\motorRx,y radius=\motorRy]
        -- ($(motorC)+(-\motorHalfLen,\motorRy)$)
        arc[start angle=90,end angle=270,x radius=\motorRx,y radius=\motorRy]
        -- cycle;

    % Back cap line
    \draw[green!45!black, line width=0.55pt]
        ($(motorC)+(-\motorHalfLen,0)$)
        ellipse [x radius=\motorRx, y radius=\motorRy];

    % Front cap line
    \draw[green!45!black, line width=0.65pt]
        ($(motorC)+(\motorHalfLen,0)$)
        ellipse [x radius=\motorRx, y radius=\motorRy];

    % -----------------------------------------------------
    % Propeller disk and blades
    % -----------------------------------------------------

    % Light projected disk, suggesting propeller plane at an angle
    \draw[gray!55, line width=0.45pt, fill=gray!8]
        (hub) ellipse [x radius=0.34, y radius=0.74];

    % Foreshortened propeller blades: x-compression mimics rotation about a vertical axis
    \begin{scope}[shift={(hub)}, xscale=0.45]
        % rear blade pair, slightly lighter
        \foreach \ang in {135,315}{
            \begin{scope}[rotate=\ang]
                \draw[draw=black!55, line width=0.45pt, fill=gray!35]
                    (0.05,-0.035)
                    .. controls (0.34,-0.08) and (0.62,-0.05) .. (0.78,0.01)
                    .. controls (0.65,0.12) and (0.34,0.10) .. (0.05,0.035)
                    -- cycle;
            \end{scope}
        }

        % front blade pair
        \foreach \ang in {45,225}{
            \begin{scope}[rotate=\ang]
                \draw[blade]
                    (0.05,-0.040)
                    .. controls (0.36,-0.085) and (0.66,-0.055) .. (0.84,0.01)
                    .. controls (0.70,0.13) and (0.36,0.11) .. (0.05,0.040)
                    -- cycle;
            \end{scope}
        }
    \end{scope}

    % Elliptic hub on top of blades, consistent with the oblique propeller view
    \draw[draw=black, fill=green!35, line width=0.7pt]
        (hub) ellipse [x radius=0.055, y radius=0.105];
    \draw[draw=green!45!black, fill=green!50, line width=0.5pt]
        (hub) ellipse [x radius=0.025, y radius=0.048];

    % Grey shaft from motor front to back of hub
    \draw[gray!65, line width=2.0pt]
        ($(motorC)+(\motorHalfLen+\motorRx,0)$) -- ($(hub)+(-0.10,0)$);
        
    % Bidirectional thrust arrows, kept on the same thrust axis
    \draw[bluearrow] (1.38,-0.80) -- (2.08,-0.80);
    \draw[bluearrow] (0.92,-0.80) -- (0.25,-0.80);

\end{scope}

    % Static map text
    \node[anchor=west] at (2.78,-0.36) {\textbf{Static thrust map}};
    \node[anchor=west] at (2.78,-0.78) {\(F=c_f\omega|\omega|\)};
    \node[anchor=west] at (2.78,-1.20) {\(\omega_d=\operatorname{sign}(F_d)\sqrt{|F_d|/c_f}\)};

    % =====================================================
    % Panel 2: dynamic realization
    % =====================================================
    \node[box, minimum width=7.2cm, minimum height=1.55cm, anchor=north west] at (0,-1.82) {};

    % Left explanatory text
    \node[anchor=west] at (0.10,-2.17) {\textbf{But \(\omega\) is a state}};
    \node[anchor=west] at (0.10,-2.56) {changing \(\omega(t)\) requires};
    \node[anchor=west] at (0.10,-2.95) {bounded physical commands};

    % Command chain to the right of the text
    \node[draw=black, rounded corners=2pt, fill=white, inner sep=2.5pt] (Fd) at (4.0,-2.55) {\(F_d(t)\)};
    \node[draw=black, rounded corners=2pt, fill=white, inner sep=2.5pt] (wd) at (5.2,-2.55) {\(\omega_d(t)\)};
    \node[draw=black, rounded corners=2pt, fill=white, inner sep=2.5pt] (cmd) at (6.55,-2.55) {\(\tau_m,i,V\)};
    \draw[bluearrow] (Fd) -- (wd);
    \draw[redarrow] (wd) -- (cmd);

    % =====================================================
    % Panel 3: singular reversal
    % =====================================================
    \node[box, minimum width=7.2cm, minimum height=2.35cm, anchor=north west] at (0,-3.63) {};

    % Left plot: linear thrust crossing
    \begin{scope}[shift={(0.66,-4.80)}]
        \draw[axis] (-0.05,0) -- (2.00,0);
        \draw[axis] (1.00,-0.50) -- (1.00,0.63);
        \draw[curve] (0.22,-0.43) -- (1.78,0.43);
        \draw[red!75!black, line width=0.8pt] (1.00,-0.10) -- (1.00,0.10);
        \draw[red!75!black, line width=0.8pt] (0.90,0) -- (1.10,0);
        \node[anchor=south] at (1.00,0.70) {\scriptsize \(F_d\)};
        \node[anchor=north] at (1.00,-0.58) {\scriptsize thrust reversal};
    \end{scope}

    % Center arrow
    \draw[redarrow] (3.04,-4.80) -- (4.04,-4.80);

    % Right plot: singular command
    \begin{scope}[shift={(4.50,-4.80)}]
        \draw[axis] (-0.05,0) -- (2.00,0);
        \draw[axis] (1.00,-0.50) -- (1.00,0.63);
        \draw[redcurve] (0.25,0.10)
            .. controls (0.65,0.10) and (0.85,0.34) .. (0.96,0.60);
        \draw[redcurve] (1.08,-0.60)
            .. controls (1.18,-0.34) and (1.40,-0.10) .. (1.78,-0.10);
        \draw[red!75!black, dashed, line width=0.7pt] (1.00,-0.50) -- (1.00,0.61);
        \node[anchor=south] at (1.00,0.70) {\scriptsize \(\tau_m,i,V\)};
        \node[anchor=north] at (1.00,-0.58) {\scriptsize singular demand};
    \end{scope}

    % Bottom message
    \node[anchor=center, font=\bfseries\footnotesize] at (3.60,-6.45)
        {static inversion \(\not\Rightarrow\) dynamic reproducibility};

    \end{tikzpicture}%
    }
    \caption{Why static thrust inversion is insufficient. 
    The static map \(F=c_f\omega|\omega|\) gives a rotor speed for each desired thrust, but rotor speed is a physical state. 
    Changing it through a thrust reversal requires bounded torque, current, and voltage over time; generic reversals make these ideal motor commands singular.}
    \label{fig:static_vs_dynamic}
\end{figure}

A key design route toward such capabilities is to increase the actuation freedom of the propeller system~\cite{ramy}. 
Tilting propellers can make a quadrotor fully actuated by adding tilt inputs to all rotors~\cite{ryll2015overactuated}, or to only one of them~\cite{aliPendu}. 
The Voliro platform demonstrated an omnidirectional hexarotor with tiltable rotors, including mechanical design, control allocation, and experiments~\cite{kamel2018voliro}. 
Morphing multirotors can switch among energy-efficient, fully actuated, and omnidirectional configurations~\cite{aboudorra2024omnimorph}. 
Bidirectional propellers provide another route: by reversing rotor speed, a fixed-pitch propeller can generate thrust in both axial directions. 
Maier integrated bidirectional fixed-pitch propellers into multirotor modeling, control, allocation, and experimental flight~\cite{maier2018bidirectional}. 
This capability is attractive for inverted flight, omnidirectional wrench generation, and contact tasks in which the desired interaction force may change sign~\cite{brescianini2018omnidirectional,park2018odar,lee2025aggressive}.

At first sight, bidirectional thrust seems to remove the main difficulty: for each desired thrust, the static thrust map provides a corresponding rotor speed, so thrust tracking appears reducible to speed tracking. 
Fig.~\ref{fig:static_vs_dynamic} shows why this reasoning is incomplete. 
Rotor speed is a physical state, not an instantaneous algebraic command: changing it through a reversal requires bounded motor torque, current, and voltage over time. 
Equivalently, planning directly in a signed-thrust coordinate, even with thrust bounds or rate limits, can hide the differential constraint imposed by the motor-propeller dynamics. 
Thus, pointwise invertibility of the thrust map does not imply dynamic reproducibility of a thrust trajectory.

This gap matters whenever a controller commands forces or accelerations that change sign. 
In aerial physical interaction, impedance, admittance, and force-tracking controllers may generate thrust references crossing zero during contact transitions, oscillatory interaction, sliding, or disturbance rejection. 
If such references are sent to bidirectional propellers without accounting for dynamic realizability, the low-level motor command can become unbounded or nonsmooth at reversal. 
The resulting tracking degradation is therefore not merely an implementation artifact, but the manifestation of an intrinsic actuator-level limitation.

This work studies this hidden dynamic constraint rigorously. 
We characterize which thrust trajectories can be reproduced exactly by bidirectional propellers under bounded motor inputs and prescribed input smoothness. 
The main conclusion is simple: generic thrust reversals are not exactly reproducible with bounded inputs, while reproducible reversals must be sufficiently flat at zero thrust. 
The same obstruction appears at the electrical level, where motor-current and motor-voltage references inherit corresponding boundedness and smoothness requirements.

%The contributions of this work are as follows:
%\begin{itemize}
%    \item We derive a normalized signed-quadratic model of a bidirectional propeller and a thrust-coordinate normal form that exposes the loss of input authority at thrust reversal.
%    \item We formulate an exact thrust-trajectory reproducibility problem under bounded motor inputs and prescribed input smoothness.
%    \item We prove necessary and sufficient conditions for reproducibility, showing that generic thrust reversals are impossible with bounded inputs, while smoother reversals require increasingly high-order flatness.
%    \item We translate the reproducibility conditions into motor-current and motor-voltage regularity requirements for a DC motor driving a bidirectional propeller.
%    \item We provide practitioner-oriented design rules for choosing thrust zero-crossing orders and generating dynamically reproducible reversal trajectories.
%    \item We experimentally validate the predicted reversal-order effects on a motor-propeller setup by comparing linear and higher-order thrust reversals.
%\end{itemize}

The contributions of this work are as follows:
\begin{itemize}
    \item We derive a normalized signed-quadratic propeller model and a thrust-coordinate normal form exposing the loss of input authority at reversal; see Section~\ref{sec:modeling_normal_form} and Fig.~\ref{fig:f2_map}.
    \item We formulate exact thrust-trajectory reproducibility under bounded motor inputs with prescribed smoothness; see Section~\ref{sec:main_result}.
    \item We prove necessary and sufficient reproducibility conditions, showing that generic thrust reversals require unbounded inputs, while reproducible reversals need increasing flatness; see Theorem~\ref{thm:in_out_feasibility}.
    \item We translate these conditions into motor-current and motor-voltage regularity requirements; see Corollary~\ref{cor:dc_motor_references}.
    \item We provide design rules for choosing thrust zero-crossing orders and generating dynamically reproducible reversal trajectories; see Table~\ref{tab:practical_rules} and \eqref{eq:sin_family}.
    \item We experimentally validate the predicted reversal-order effects on a motor-propeller setup; see Figs.~\ref{fig:experimental_setup} and~\ref{fig:dc_exp_cases}.
\end{itemize}

%%%%%%%%%%%%%%%%%%%%%%%%%%%%%%%%%%%%%%%%
\section{System Model and Normal Form}
\label{sec:modeling_normal_form}
%%%%%%%%%%%%%%%%%%%%%%%%%%%%%%%%%%%%%%%%

\subsubsection{Bidirectional propeller dynamics}

Consider a fixed-pitch propeller driven by a motor that can rotate in both directions. 
Let \(\omega\in\mathbb{R}\) denote the signed rotor speed, \(J>0\) the rotor-propeller inertia, \(\tau_m\) the motor torque, and \(c_\tau>0\) the aerodynamic drag coefficient. 
Under the usual quasi-steady approximation, the aerodynamic drag torque is modeled as \(c_\tau\omega|\omega|\), yielding
\begin{equation}
    J\dot{\omega} = \tau_m - c_\tau\omega|\omega|.
    \label{eq:omega_physical}
\end{equation}
With the normalized motor input \(u:=\tau_m/J\) and the normalized drag coefficient \(k:=c_\tau/J>0\), \eqref{eq:omega_physical} becomes
\begin{equation}
    \dot{\omega} = u-k\omega|\omega|.
    \label{eq:omega_dyn}
\end{equation}
The axial thrust generated by an ideal symmetric bidirectional propeller is modeled as
\begin{equation}
    F = c_f\omega|\omega|,
    \label{eq:thrust_model}
\end{equation}
where \(c_f>0\) is the thrust coefficient. 
Thus, the same signed-quadratic structure appears in both the aerodynamic drag and the generated thrust. 
In particular, any thrust reversal must pass through \(F=0\), equivalently \(\omega=0\). 
At this point, the slope of the map \(\omega\mapsto c_f\omega|\omega|\) vanishes, so the propeller cannot behave as an ideal signed force source through zero thrust.

\subsubsection{Minimal thrust-channel model}

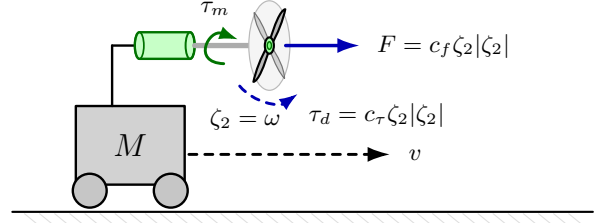
\begin{figure}[t]
    \centering
    \begin{tikzpicture}[
        scale=0.8,
        line cap=round,
        line join=round,
        >=Latex,
        font=\small,
        ground/.style={black, line width=0.8pt},
        body/.style={draw=black, line width=0.9pt, fill=gray!35},
        wheel/.style={draw=black, line width=0.9pt, fill=gray!45},
        motor/.style={draw=green!45!black, line width=0.9pt, fill=green!25},
        blade/.style={draw=black, line width=0.8pt, fill=gray!55},
        thrust/.style={blue!70!black, line width=1.3pt, -{Latex[length=3mm,width=2mm]}},
        velocity/.style={black, dashed, line width=1.1pt, -{Latex[length=3mm,width=2mm]}},
        torque/.style={green!45!black, line width=1.1pt, -{Latex[length=2.5mm,width=1.8mm]}},
        drag/.style={blue!70!black, dashed, line width=1.1pt, -{Latex[length=2.5mm,width=1.8mm]}},
        label/.style={font=\small}
    ]

    % Ground line and hatching
    \draw[ground] (-0.4,0) -- (9.2,0);
    \foreach \x in {-0.2,0.05,...,8.9}
        \draw[gray!25, line width=0.4pt] (\x,-0.03) -- ++(0.22,-0.18);

    % Cart body
    \draw[body] (0.65,0.45) rectangle (2.45,1.75);
    \node[font=\large] at (1.55,1.10) {$M$};

    % Wheels
    \draw[wheel] (0.88,0.35) circle (0.28);
    \draw[wheel] (2.22,0.35) circle (0.28);

    % Support
    \draw[black, line width=0.9pt] (1.25,1.75) -- (1.25,2.75);
    \draw[black, line width=0.9pt] (1.25,2.75) -- (1.65,2.75);

    % -----------------------------------------------------
    % Motor as an oblique cylinder
    % -----------------------------------------------------
    \coordinate (motorC) at (2.12,2.75);
    \coordinate (hub) at (3.85,2.75);

    \def\motorHalfLen{0.42}
    \def\motorRx{0.075}
    \def\motorRy{0.22}

    % Motor body
    \draw[draw=green!45!black, line width=0.9pt, fill=green!22]
        ($(motorC)+(-\motorHalfLen,-\motorRy)$)
        -- ($(motorC)+(\motorHalfLen,-\motorRy)$)
        arc[start angle=-90,end angle=90,x radius=\motorRx,y radius=\motorRy]
        -- ($(motorC)+(-\motorHalfLen,\motorRy)$)
        arc[start angle=90,end angle=270,x radius=\motorRx,y radius=\motorRy]
        -- cycle;

    % Back cap
    \draw[green!45!black, line width=0.65pt]
        ($(motorC)+(-\motorHalfLen,0)$)
        ellipse [x radius=\motorRx, y radius=\motorRy];

    % Front cap
    \draw[green!45!black, line width=0.75pt]
        ($(motorC)+(\motorHalfLen,0)$)
        ellipse [x radius=\motorRx, y radius=\motorRy];

    % Grey shaft from motor front to back of hub
    \draw[gray!65, line width=2.0pt]
        ($(motorC)+(\motorHalfLen+\motorRx,0)$) -- ($(hub)+(-0.10,0)$);

    % -----------------------------------------------------
    % Oblique propeller disk and blades
    % -----------------------------------------------------

    % Light projected disk, suggesting propeller plane seen from an angle
    \draw[gray!55, line width=0.45pt, fill=gray!8]
        (hub) ellipse [x radius=0.34, y radius=0.74];

    % Foreshortened propeller blades
    \begin{scope}[shift={(hub)}, xscale=0.45]
        % Rear blade pair
        \foreach \ang in {135,315}{
            \begin{scope}[rotate=\ang]
                \draw[draw=black!55, line width=0.45pt, fill=gray!35]
                    (0.05,-0.035)
                    .. controls (0.34,-0.08) and (0.62,-0.05) .. (0.78,0.01)
                    .. controls (0.65,0.12) and (0.34,0.10) .. (0.05,0.035)
                    -- cycle;
            \end{scope}
        }

        % Front blade pair
        \foreach \ang in {45,225}{
            \begin{scope}[rotate=\ang]
                \draw[blade]
                    (0.05,-0.040)
                    .. controls (0.36,-0.085) and (0.66,-0.055) .. (0.84,0.01)
                    .. controls (0.70,0.13) and (0.36,0.11) .. (0.05,0.040)
                    -- cycle;
            \end{scope}
        }
    \end{scope}

    % Elliptic hub on top of blades
    \draw[draw=black, fill=green!35, line width=0.8pt]
        (hub) ellipse [x radius=0.075, y radius=0.135];
    \draw[draw=green!45!black, fill=green!50, line width=0.55pt]
        (hub) ellipse [x radius=0.035, y radius=0.062];

    % -----------------------------------------------------
    % Torque, drag, thrust, and velocity annotations
    % -----------------------------------------------------

    % Motor torque arrow
    \begin{scope}
        \coordinate (motorFront) at ($(motorC)+(\motorHalfLen+\motorRx+0.5,0)$);
    
        \draw[torque]
            ($(motorFront)+(-0.02,-0.30)$)
            .. controls ($(motorFront)+(-0.56,-0.28)$)
                     and ($(motorFront)+(-0.26,0.28)$)
            .. ($(motorFront)+(0.22,0.1)$);
    
        \node[label, anchor=south] at ($(motorFront)+(-0.16,0.35)$) {$\tau_m$};
    \end{scope}

    % Drag torque arrow near propeller
    \draw[drag] (3.36,2.08) arc[start angle=220,end angle=305,radius=0.72];
    \node[label, anchor=west] at (4.35,1.60)
        {$\tau_d=c_\tau\zeta_2|\zeta_2|$};

    % Rotor speed label
    \node[label, anchor=south] at (3.45,1.20) {$\zeta_2=\omega$};

    % Thrust arrow and label
    \draw[thrust] (4.12,2.75) -- (5.35,2.75);
    \node[label, anchor=west] at (5.48,2.75)
        {$F=c_f\zeta_2|\zeta_2|$};

    % Velocity arrow and label
    \draw[velocity] (2.52,0.95) -- (5.85,0.95);
    \node[label, anchor=west] at (6.00,0.95) {$v$};

    \end{tikzpicture}
    \caption{Minimal thrust-channel abstraction. 
    A bidirectional propeller driven by the motor torque \(\tau_m\) applies the signed-quadratic thrust \(F=c_f\zeta_2|\zeta_2|\) to the one-dimensional mass.}
    \label{fig:cartmodel}
\end{figure}

To isolate this actuator-level effect, consider the minimal one-dimensional channel shown in Fig.~\ref{fig:cartmodel}. 
A mass \(M>0\) with velocity \(v\) is actuated by the bidirectional propeller. Applying Newton's decond law gives
\begin{equation}
    M\dot v = F = c_f\omega|\omega|.
    \label{eq:mass_physical}
\end{equation}
With \(\zeta_1:=Mv/c_f\) and \(\zeta_2:=\omega\), \eqref{eq:omega_dyn} and \eqref{eq:mass_physical} yield
\begin{equation}
\begin{aligned}
    \dot{\zeta}_1 &= \zeta_2|\zeta_2|,\\
    \dot{\zeta}_2 &= u-k\zeta_2|\zeta_2|.
\end{aligned}
\label{eq:Sigma_zeta}
\end{equation}
where $k>0$. This model does not describe a full aerial robot. 
Rather, it captures the single thrust/acceleration channel in which the signed-quadratic reversal limitation appears. 
The normalization removes the parameters \(M\) and \(c_f\), while preserving the intrinsic dependence of thrust on \(\zeta_2|\zeta_2|\).

\subsubsection{Thrust-coordinate normal form}

Define \(x_1:=\zeta_1\) and the normalized thrust coordinate \(x_2:=\zeta_2|\zeta_2|\). 
Then \(x_2\) is the normalized force applied to the mass, \(\dot x_1=x_2\), and \(\zeta_2=\operatorname{sign}(x_2)\sqrt{|x_2|}\). 
While the map \(\zeta_2\mapsto x_2\) is continuous and one-to-one, it is not a diffeomorphism at \(x_2=0\). 
Differentiating \(x_2=\zeta_2|\zeta_2|\) along \eqref{eq:Sigma_zeta} gives the normal form
\begin{equation}
\begin{aligned}
    \dot{x}_1 &= x_2,\\
    \dot{x}_2 &=
        -2k\,\operatorname{sign}(x_2)|x_2|^{3/2}
        +2\sqrt{|x_2|}\,u,\\
    y &= x_1.
\end{aligned}
\label{eq:sys2}
\end{equation}
Here \(y\) is the normalized output whose derivative is the normalized thrust, i.e., \(\dot y=x_2\). 
 \eqref{eq:sys2} is an input-affine, \(\dot x=f(x)+g(x)u\), with input gain \(g_2(x_2)=2\sqrt{|x_2|}\) in the thrust dynamics, which is strictly positive for \(x_2\neq0\) and vanishes at \(x_2=0\). For \(x_2\neq0\), one may formally invert the second equation in \eqref{eq:sys2} to assign a desired thrust derivative. 
However, this inverse contains the factor \(1/\sqrt{|x_2|}\), which becomes singular at zero thrust. 
Moreover, the drift term \(x_2\mapsto\operatorname{sign}(x_2)|x_2|^{3/2}=x_2\sqrt{|x_2|}\) is in \(\mathcal C^1(\mathbb R)\setminus\mathcal C^2(\mathbb R)\), whereas the input gain \(x_2\mapsto\sqrt{|x_2|}\) is in \(\mathcal C^0(\mathbb R)\setminus\mathcal C^1(\mathbb R)\). 
Thus, the drift remains continuously differentiable at zero thrust while the input gain is only continuous and collapses to zero. 
Fig.~\ref{fig:f2_map} visualizes this structural degeneracy, which is the origin of the reproducibility constraints derived in Section~\ref{sec:main_result}.

\begin{figure}[t]
    \centering
    \includegraphics[width=0.85\columnwidth]{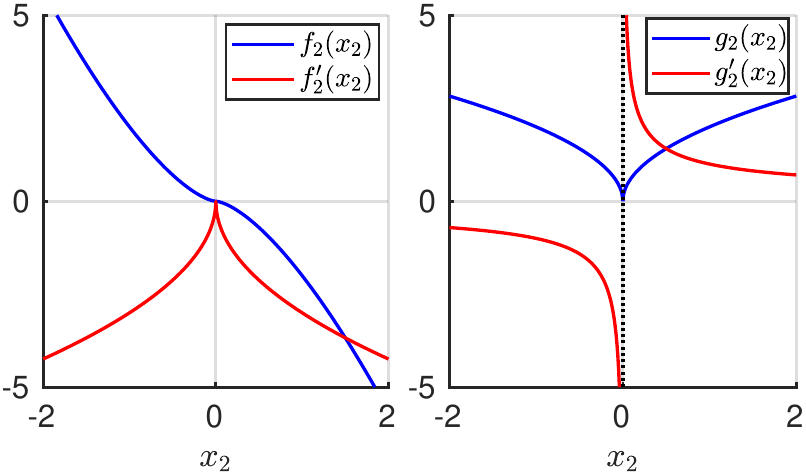}
    \caption{Normal-form drift and input gain as functions of the normalized thrust \(x_2\). 
    The input gain \(g_2(x_2)=2\sqrt{|x_2|}\) vanishes at \(x_2=0\), i.e., at thrust reversal.}
    \label{fig:f2_map}
\end{figure}

%%%%%%%%%%%%%%%%%%%%%%%%%%%%%%%%%%%%%%%%
\section{Reproducible Trajectories and Thrust Reversals}
\label{sec:main_result}
%%%%%%%%%%%%%%%%%%%%%%%%%%%%%%%%%%%%%%%%

We characterize the output trajectories that system \eqref{eq:sys2} can exactly reproduce under bounded motor inputs. 
Recall that \(y=x_1\) and \(\dot y=x_2\), so zeros of \(\dot y\) correspond to normalized thrust reversals. 
Let \(\mathcal T:=[t_i,t_f]\) be compact and let the admissible input set be \(\mathbb U:=\{u\in\mathbb R:-\underline u\le u\le\overline u\}\), with \(0<\underline u<\overline u<+\infty\). 
For an integer \(m\ge0\), let \(\mathcal C^m(\mathcal T,\mathbb U)\) denote the set of \(\mathcal C^m\) functions \(u:\mathcal T\to\mathbb R\) satisfying \(u(t)\in\mathbb U\) for all \(t\in\mathcal T\). 
For \(\dot y\neq0\), solving \eqref{eq:sys2} for \(u\) gives the unique inverse
\begin{equation}
    u=\frac{\ddot y}{2\sqrt{|\dot y|}}+k\dot y .
    \label{eq:inv_u}
\end{equation}
Thus, away from thrust reversal, input bounds impose the following feasibility set
\begin{equation}
\mathbb S :=
\left\{
(\dot y,\ddot y)\in\mathbb R^2:
\dot y\neq0,\ 
-\underline u
\le
\frac{\ddot y}{2\sqrt{|\dot y|}}+k\dot y
\le
\overline u
\right\}.
\label{eq:feasibility_set}
\end{equation}

Fig.~\ref{fig:trjphase} visualizes this construction. 
Away from thrust reversal, i.e., for \(\dot y_d\neq0\), the bounded input interval \(\mathbb U\) maps through \eqref{eq:inv_u} to a feasible region in the \((\dot y_d,\ddot y_d)\)-plane. 
The singular line \(\dot y_d=0\) is not part of \(\mathbb S\); the behavior of trajectories approaching this line is classified by Theorem~\ref{thm:in_out_feasibility}. We first state the elementary regularity fact used at reversal instants.
\begin{figure}[t]
    \centering
    \includegraphics[width=1\columnwidth]{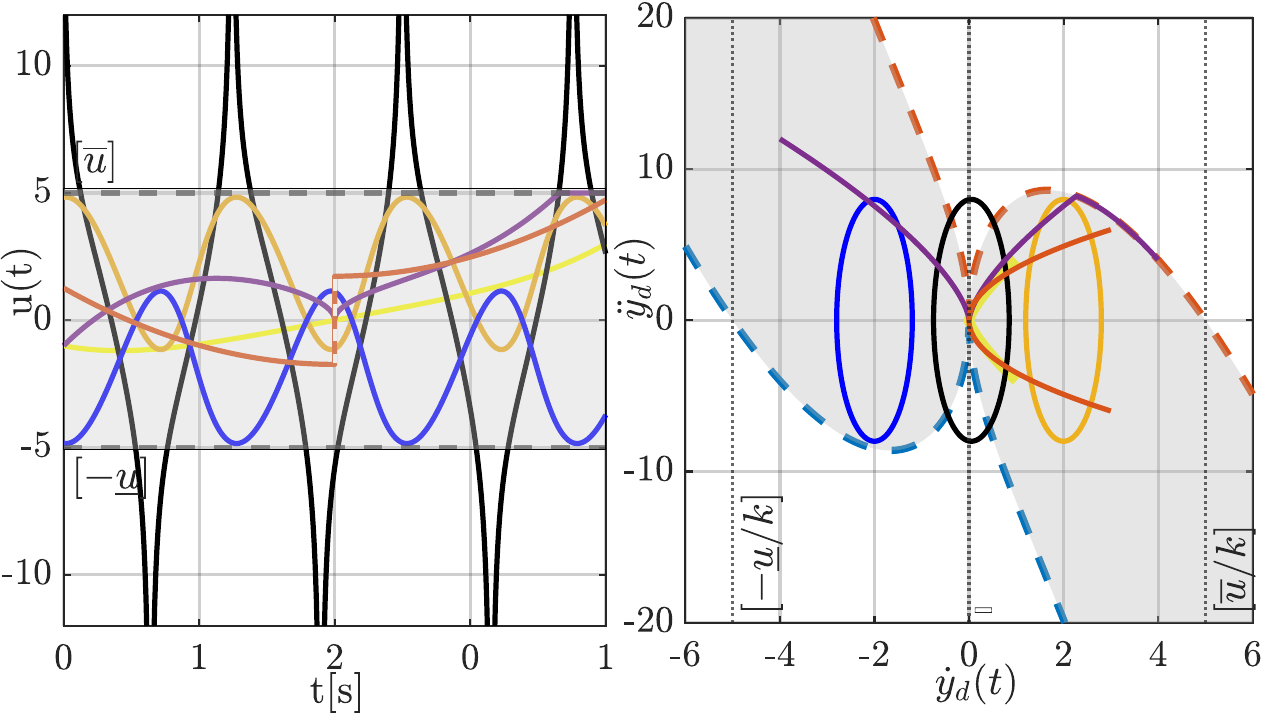}
    \caption{Input bounds dictate the desired trajectory $y_d$ feasibility away from thrust reversal. 
    The interval \(\mathbb U\) of admissible motor inputs maps, through the inverse \eqref{eq:inv_u}, to the shaded feasible set \(\mathbb S\) in the \((\dot y_d,\ddot y_d)\)-plane. 
    The boundary curves correspond to input saturation. 
    The line \(\dot y_d=0\) is excluded from \(\mathbb S\), and Theorem~\ref{thm:in_out_feasibility} characterizes which trajectories can approach this line with bounded, possibly nonsmooth, inputs.}
    \label{fig:trjphase}
\end{figure}

\begin{lemma}[Regularity of reversal monomials]
\label{lem:reversal_monomial}
Let \(r\ge2\) be an integer and define, for \(s\neq0\), \(\phi_r(s):=\operatorname{sign}(s)^{r-2}|s|^{(r-3)/2}\). 
Then \(\phi_2\) is unbounded, \(\phi_3\) is bounded but discontinuous, and \(\phi_r\) is continuous for every \(r\ge4\). 
Moreover, for any integer \(m\ge0\), \(\phi_r\in\mathcal C^m\) at \(s=0\) if either \(r\ge R_m\), where
\begin{equation}
R_m=
\begin{cases}
2m+4, & m\ \mathrm{even},\\
2m+3, & m\ \mathrm{odd},
\end{cases}
\label{eq:Rm}
\end{equation}
or \(r=4q+1\) for some integer \(q\ge1\). 
\end{lemma}

\begin{proof}
For \(r=2\), the exponent \((r-3)/2\) is negative, hence \(\phi_2\) is unbounded. 
For \(r=3\), \(\phi_3(s)=\operatorname{sign}(s)\), hence \(\phi_3\) is bounded but discontinuous. 
For \(r\ge4\), the exponent is positive, hence \(\phi_r(s)\to0\) as \(s\to0\), so \(\phi_r\) is continuous.

If \(r\) is even, then \(\operatorname{sign}(s)^{r-2}=1\) and \(\phi_r(s)=|s|^{(r-3)/2}\), with a noninteger exponent. 
It is \(\mathcal C^m\) at the origin iff \((r-3)/2>m\), equivalently \(r\ge2m+4\). 
If \(r\) is odd, then \((r-3)/2\) is an integer. 
When \(r=4q+1\), \(\phi_r(s)=s^{(r-3)/2}\), hence it is smooth. 
When \(r=4q+3\), \(\phi_r(s)=s|s|^{(r-5)/2}\), which is \(\mathcal C^{(r-5)/2}\) but not \(\mathcal C^{(r-3)/2}\). 
This yields the threshold \eqref{eq:Rm} and the exceptional smooth cases \(r=4q+1\).
\end{proof}
Now we introduce, in Theorem \ref{thm:in_out_feasibility}, the necessary and sufficient thrust
reproducibility conditions.

\begin{theorem}[Output-to-input correspondence under bounded inputs]
\label{thm:in_out_feasibility}
Consider system \eqref{eq:sys2} with \(u(t)\in\mathbb U\). 
Let \(y_d:\mathcal T\to\mathbb R\) be a desired output trajectory. 
{If \(y_d\) is not constant, define}
\[
\mathcal T_s:=\{t\in\mathcal T:\dot y_d(t)=0\},
\]
{and assume that every \(t_s\in\mathcal T_s\) is isolated and that \(y_d\) is sufficiently smooth at each \(t_s\) so that, whenever needed, \(r(t_s):=\min\{i>1:y_d^{(i)}(t_s)\neq0\}\) is finite.}

Then the following statements hold.

\begin{enumerate}
\item \textbf{Constant outputs.}
A constant output \(y_d(t)\equiv\bar y\) is reproducible iff \(y(t_i)=\bar y\) and \(\dot y(t_i)=0\). 
In this case \(x_2(t)\equiv0\), and any input \(u(t)\in\mathbb U\) can reproduce \(y_d\).

\item \textbf{Minimal regularity.}
Every output produced by \eqref{eq:sys2} with bounded input satisfies \(y\in\mathcal C^1(\mathcal T)\). 
Hence, any \(y_d\notin\mathcal C^1(\mathcal T)\) is not reproducible.

\item \textbf{No thrust reversal.}
{For a nonconstant trajectory with} \(\mathcal T_s=\emptyset\), any \(y_d\in\mathcal C^{m+2}(\mathcal T)\) is reproducible by a unique input \(u\in\mathcal C^m(\mathcal T,\mathbb U)\) iff \(y_d(t_i)=y(t_i)\), \(\dot y_d(t_i)=\dot y(t_i)\), and \((\dot y_d(t),\ddot y_d(t))\in\mathbb S\) for all \(t\in\mathcal T\).

\item \textbf{With thrust reversal.}
{For a nonconstant trajectory with} \(\mathcal T_s\neq\emptyset\), {the following hold.}
\begin{enumerate}
\item If \(r(t_s)=2\) for some \(t_s\in\mathcal T_s\), then \(y_d\) is not reproducible by any bounded input.

\item If \(r(t_s)\ge3\) for all \(t_s\in\mathcal T_s\), the feasibility condition holds on \(\mathcal T\setminus\mathcal T_s\), and the finite limiting values at the crossings belong to \(\mathbb U\), then bounded reproduction is possible. If \(r(t_s)=3\) at some crossing, every reproducing input is necessarily discontinuous there. 
At such a crossing, boundedness requires \(\sqrt{|y_d^{(3)}(t_s)|/2}\le\min\{\underline u,\overline u\}\), together with the feasibility condition \((\dot y_d(t),\ddot y_d(t))\in\mathbb S\) for all \(t\in\mathcal T\setminus\mathcal T_s\).

\item Let \(m\ge0\). 
A trajectory \(y_d\in\mathcal C^{m+2}(\mathcal T)\) is reproducible by a unique input \(u\in\mathcal C^m(\mathcal T,\mathbb U)\) iff the initial conditions match, \((\dot y_d(t),\ddot y_d(t))\in\mathbb S\) for all \(t\in\mathcal T\setminus\mathcal T_s\), and, at every \(t_s\in\mathcal T_s\), either \(r(t_s)\ge R_m\) or \(r(t_s)=4q+1\) for some integer \(q\ge1\). 
\end{enumerate}
\end{enumerate}
\end{theorem}

\begin{proof}
Item 1 follows directly from \(\dot y=x_2\). 
If \(y_d\equiv\bar y\), exact reproduction requires \(x_2\equiv0\), which is invariant for \eqref{eq:sys2} independently of \(u\). 
Thus the stated initial conditions are necessary and sufficient.

For item 2, the right-hand side of \eqref{eq:sys2} is continuous in \(x\) for every bounded \(u\). 
Hence \(x_2(\cdot)\) is continuous on \(\mathcal T\), and since \(\dot y=x_2\), every reproducible output belongs to \(\mathcal C^1(\mathcal T)\).

For item 3, if \(\dot y_d\neq0\) on \(\mathcal T\), the input channel in \eqref{eq:sys2} does not vanish along the desired trajectory. 
Solving the second equation of \eqref{eq:sys2} with \(x_2=\dot y_d\) gives the unique inverse \eqref{eq:inv_u}. 
The condition \(u(t)\in\mathbb U\) is exactly \((\dot y_d(t),\ddot y_d(t))\in\mathbb S\). 
Since \(y_d\in\mathcal C^{m+2}(\mathcal T)\) and \(\dot y_d\) is bounded away from zero, \eqref{eq:inv_u} belongs to \(\mathcal C^m(\mathcal T,\mathbb U)\).

For item 4, {the constant-output case has already been handled in item 1; hence the zeros of \(\dot y_d\) are isolated by assumption.}
Fix \(t_s\in\mathcal T_s\) and set \(s:=t-t_s\). 
On a punctured neighborhood of \(t_s\), the reproducing input is necessarily given by \eqref{eq:inv_u}. 
Since \(k\dot y_d(t_s)=0\), the limiting behavior is governed by \(\ddot y_d/(2\sqrt{|\dot y_d|})\). 
If \(r:=r(t_s)\), Taylor expansion gives \(\dot y_d(t)=a s^{r-1}+\mathcal O(s^r)\) and \(\ddot y_d(t)=(r-1)a s^{r-2}+\mathcal O(s^{r-1})\), with \(a:=y_d^{(r)}(t_s)/(r-1)!\neq0\). 
Therefore, on the punctured neighborhood,
\[
    u(t)=c\,\operatorname{sign}(s)^{r-2}|s|^{(r-3)/2}+o\!\left(|s|^{(r-3)/2}\right),
\]
where \(c\neq0\). 
Lemma~\ref{lem:reversal_monomial} then gives the boundedness, discontinuity, and \(\mathcal C^m\) classifications. 
For \(r=3\), the one-sided limits have magnitude \(\sqrt{|y_d^{(3)}(t_s)|/2}\), which yields the stated bound. 
The same inverse is unique away from \(\mathcal T_s\), and the \(\mathcal C^m\) conditions at the isolated crossings give uniqueness of the continuous extension when such an extension exists.
\end{proof}

\begin{remark}[Interpretation of Theorem~\ref{thm:in_out_feasibility}]
A generic thrust reversal corresponds to \(\dot y_d(t_s)=0\) and \(\ddot y_d(t_s)\neq0\), i.e., \(r(t_s)=2\). 
Theorem~\ref{thm:in_out_feasibility} states that such a reversal cannot be reproduced with bounded motor input. 
If the crossing is flatter, with \(r(t_s)=3\), bounded reproduction is possible only with an input jump. 
Continuous or smoother motor inputs require higher-order flatness at the reversal, as quantified by \(R_m\), with the precise parity-dependent cases stated in Lemma~\ref{lem:reversal_monomial}.
\end{remark}

%%%%%%%%%%%%%%%%%%%%%%%%%%%%%%%%%%%%%%%%
\section{Consequences for Motor Commands and Interaction References}
\label{sec:consequences}
%%%%%%%%%%%%%%%%%%%%%%%%%%%%%%%%%%%%%%%%

While Theorem~\ref{thm:in_out_feasibility} characterizes reproducibility at the normalized thrust level, we now translate this result to motor command references, and then discuss its implication for interaction-control references requiring thrust reversals.

\subsection{Motor current and voltage references}

Consider a DC motor with viscous friction and quadratic propeller drag,
\begin{align}
    J\dot{\omega}
        &=k_t i-B\omega-c_\tau\omega|\omega|,
        \label{eq:dc_mechanical}\\
    L_a\dot{i}
        &=V-R_ai-k_e\omega,
        \label{eq:dc_electrical}
\end{align}
where \(J>0\) is the motor-propeller inertia, \(B\ge0\) the viscous-friction coefficient, \(c_\tau\ge0\) the quadratic drag coefficient, \(k_t>0\) the motor torque constant, \(k_e>0\) the back-emf constant, \(R_a>0\) the armature resistance, and \(L_a>0\) the armature inductance. 
Let \(i\), \(V\), and \(\omega\) denote armature current, motor voltage, and rotor speed.

For a desired output \(y_d\), exact reproduction of the normalized thrust \(\dot y_d\) requires \(\omega_d=\operatorname{sign}(\dot y_d)\sqrt{|\dot y_d|}\). 
The corresponding current and voltage references are
\begin{align}
    i_d
        &=
        a_1\dot{\omega}_d
        +a_2\omega_d
        +a_3\omega_d|\omega_d|,
        \label{eq:current_reference}\\
    V_d
        &=
        b_1\ddot{\omega}_d
        +b_2\dot{\omega}_d
        +b_3\omega_d
        +b_4\omega_d|\omega_d|
        +b_5|\omega_d|\dot{\omega}_d,
        \label{eq:voltage_reference}
\end{align}
where \(a_1=J/k_t\), \(a_2=B/k_t\), \(a_3=c_\tau/k_t\), \(b_1=L_aJ/k_t\), \(b_2=(L_aB+R_aJ)/k_t\), \(b_3=k_e+R_aB/k_t\), \(b_4=R_ac_\tau/k_t\), and \(b_5=2L_ac_\tau/k_t\).

\begin{corollary}[Motor-command regularity]
\label{cor:dc_motor_references}
The current reference \(i_d\) has the same boundedness and regularity classification as the motor-torque input in Theorem~\ref{thm:in_out_feasibility}. 
The voltage reference \(V_d\) requires one additional time derivative: \(V_d\in\mathcal C^m\) iff the conditions of Theorem~\ref{thm:in_out_feasibility} hold with \(m\) replaced by \(m+1\). 
In particular, a continuous current reference requires zero-crossing order \(r\ge4\), while a continuous voltage reference requires zero-crossing order \(r\ge5\).
\end{corollary}

\begin{table*}[t]
\centering
\caption{Practical interpretation of thrust zero-crossing order. 
Here \(x_{2d}:=\dot y_d\) is the desired normalized thrust and \(x_{2d}(t)=\alpha(t-t_s)^p+o(|t-t_s|^p)\), with \(\alpha\neq0\), near an isolated zero \(t_s\). 
The theorem order is \(r=p+1\). 
Odd \(p\) corresponds to a sign-changing thrust reversal; even \(p\) corresponds to touching zero without changing thrust direction.}
\label{tab:practical_rules}
\resizebox{\textwidth}{!}{%
\begin{tabular}{p{2.4cm} p{1cm} p{1cm} p{1.3cm} p{1.3cm} p{1.3cm} p{5.8cm}}
\toprule
Desired thrust near zero & Thrust order \(p\) & Theorem order \(r\)  & Motor  torque \(u\)& Current \(i_d\) & Voltage \(V_d\) & Practical meaning\\
\midrule
\(x_{2d}\sim (t-t_s)\)
& \(1\)
& \(2\)
& unbounded
& unbounded
& unbounded
& Linear sign-changing thrust reversal. This includes ordinary sinusoidal crossings. Exact reproduction is impossible with bounded motor commands.\\

\addlinespace[2pt]
\(x_{2d}\sim (t-t_s)^2\)
& \(2\)
& \(3\)
& bounded, discontinuous
& bounded, discontinuous
& not classically bounded
& The thrust reaches zero and returns with the same sign. It avoids torque blow-up, but the ideal torque/current is still discontinuous, so this is unsuitable when smooth commands are required.\\

\addlinespace[2pt]
\(x_{2d}\sim (t-t_s)^3\)
& \(3\)
& \(4\)
& continuous
& continuous
& unbounded
& Minimum sign-changing reversal order for continuous torque and current. It is still insufficient for bounded voltage feedforward.\\

\addlinespace[2pt]
\(x_{2d}\sim (t-t_s)^4\)
& \(4\)
& \(5\)
& regular
& regular
& continuous
& The thrust reaches zero and returns with the same sign. This case is useful as a regularity benchmark and shows that zero thrust can be passed smoothly if no sign change is required.\\

\addlinespace[2pt]
\(x_{2d}\sim (t-t_s)^5\)
& \(5\)
& \(6\)
& at least \(\mathcal C^1\)
& at least \(\mathcal C^1\)
& continuous
& Required minimum sign-changing reversal order for continuous voltage. This is the first odd-order crossing that avoids the voltage singularity.\\

\addlinespace[2pt]
\(x_{2d}\sim (t-t_s)^7\)
& \(7\)
& \(8\)
& at least \(\mathcal C^2\)
& at least \(\mathcal C^2\)
& at least \(\mathcal C^1\)
& Higher-order sign-changing thrust reversal. Provides additional smoothness margin and reduces localized reversal stress.\\

\addlinespace[2pt]
\(x_{2d}=\) \phantom{rrrtttttrrrttttttt} ~\(A\sin^{2q+1}(\Omega t+\phi)\)
& \(2q+1\)
& \(2q+2\)
& improves with \(q\)
& improves with \(q\)
& improves with \(q\)
& Simple periodic family for implementation. \(q=0\) is the unreproducible case; \(q=1\) gives continuous torque/current, whereas \(q\ge2\) gives continuous voltage.\\
\bottomrule
\end{tabular}%
}
\end{table*}

\begin{proof}
Solving \eqref{eq:dc_mechanical} for the current gives \(i_d=(J/k_t)\dot{\omega}_d+(B/k_t)\omega_d+(c_\tau/k_t)\omega_d|\omega_d|\), which is \eqref{eq:current_reference}. 
Using \(d(\omega_d|\omega_d|)/dt=2|\omega_d|\dot{\omega}_d\) and substituting \(\dot i_d\) into \(V_d=L_a\dot i_d+R_ai_d+k_e\omega_d\) gives \eqref{eq:voltage_reference}. 
Since \(L_a>0\), the voltage reference contains one additional derivative of the current reference, yielding the stated regularity shift.
\end{proof}

Combining Theorem~\ref{thm:in_out_feasibility} and Corollary~\ref{cor:dc_motor_references} gives a practical rule in terms of the desired normalized thrust. 
Let \(x_{2d}:=\dot y_d\), and suppose that near an isolated zero \(t_s\), \(x_{2d}(t)=\alpha(t-t_s)^p+o(|t-t_s|^p)\), with \(\alpha\neq0\). 
Then \(p\) is the zero-crossing order of the desired thrust and the theorem order is \(r=p+1\). 
Table~\ref{tab:practical_rules} summarizes the corresponding motor-command implications. 
Odd \(p\) gives a sign-changing thrust reversal, while even \(p\) only brings the thrust to zero and back without reversing its direction.

\subsubsection{Practical reference generation}

Table~\ref{tab:practical_rules} can be used as a reference-shaping rule. 
If a thrust reversal is expected at \(t_s\), the desired normalized thrust \(x_{2d}\) should be made locally flat enough at zero. 
A simple local construction is \(x_{2d}(t)=\alpha(t-t_s)^p\chi(t)\), where \(\alpha\neq0\), \(p\) is selected according to the desired motor-command regularity in Table~\ref{tab:practical_rules}, and \(\chi\) is a smooth blending function used to connect the local reversal profile to the rest of the trajectory. 

For periodic or oscillatory thrust references, a convenient family is
\begin{equation}
    x_{2d}(t)=A\sin^{2q+1}(\Omega t+\phi),
    \label{eq:sin_family}
\end{equation}
where \(A\), \(\Omega\), and \(\phi\) set the amplitude, frequency, and phase. 
The case \(q=0\) gives sinusoidal crossing and corresponds to the unreproducible linear reversal. 
The case \(q=1\), i.e., \(x_{2d}=A\sin^3(\Omega t+\phi)\), is the minimum sign-changing reversal order for continuous torque and current, but it does not remove the voltage singularity. 
Choosing \(q\ge2\), for instance \(x_{2d}=A\sin^5(\Omega t+\phi)\), gives a quintic or higher-order reversal and is the practical minimum for continuous voltage references. 
These constructions do not increase the physical actuator limits; they shape the reference so that the ideal inverse no longer demands singular motor commands at thrust reversal. 
Thus, a simple rule of thumb is: avoid linear thrust crossings, use at least cubic crossings for continuous torque/current, and use at least quintic crossings when voltage continuity is required.

This also shows why thrust-space bounds alone are insufficient. A reference can be bounded, and even rate-limited, in the thrust coordinate while still requiring singular motor commands after conversion to rotor speed. The relevant design quantity at reversal is the zero-crossing order \(p\), in Table~\ref{tab:practical_rules}, not only the amplitude or slope away from zero.

\subsubsection{Ideal motor-command profiles}

The motor-command implications of Table~\ref{tab:practical_rules} are visualized in Fig.~\ref{fig:dc_reference_cases}. 
The figure reports the ideal voltage, current, speed, and thrust references obtained from \eqref{eq:current_reference}--\eqref{eq:voltage_reference} for six representative zero-crossing orders. 
The linear case corresponds to \(p=1\), hence \(r=2\), and produces singular ideal current and voltage references at reversal. 
Increasing the crossing order flattens the desired thrust near zero and progressively regularizes the required motor commands. 
These profiles are ideal feedforward references: in a real implementation, saturation and sampling replace singular commands with finite peaks and tracking errors, as shown experimentally in Section~\ref{sec:exp}.

\begin{figure*}[t]
    \centering
    \begin{subfigure}{0.49\textwidth}
        \centering
        \includegraphics[width=1.05\linewidth]{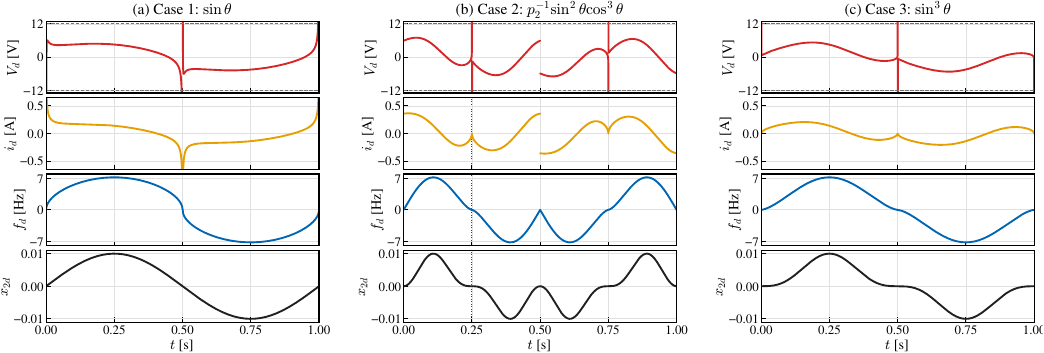}
        \caption{Cases~1--3.}
        \label{fig:dc_cases_1_3}
    \end{subfigure}
    \hfill
    \begin{subfigure}{0.49\textwidth}
        \centering
        \includegraphics[width=1.05\linewidth]{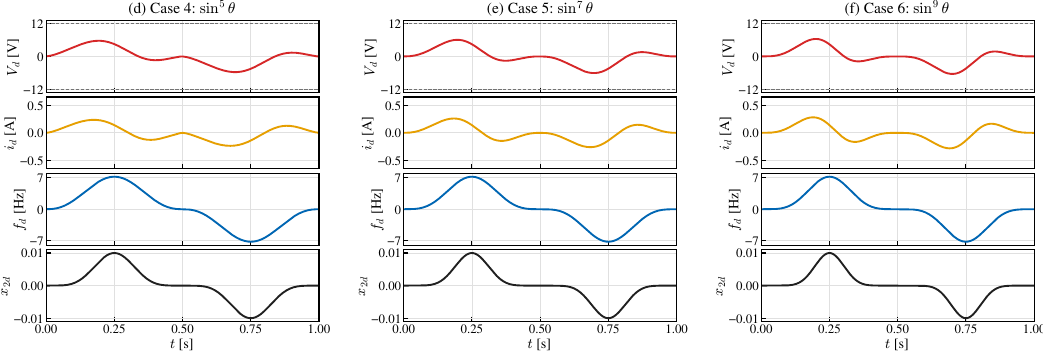}
        \caption{Cases~4--6.}
        \label{fig:dc_cases_4_6}
    \end{subfigure}
    \caption{Ideal voltage, current, speed, and thrust references for six zero-crossing orders. 
    Linear thrust reversal requires singular ideal motor commands, whereas flatter reversals progressively improve boundedness and smoothness.}
    \label{fig:dc_reference_cases}
\end{figure*}

\subsection{Interaction-control references}

The same obstruction appears when a high-level interaction controller commands a desired thrust or acceleration that changes sign. 
To make this point explicit, consider the normalized interaction channel
\begin{equation}
\begin{aligned}
    \dot{x}_0 &= x_1,\\
    \dot{x}_1 &= x_2+f_e(t),\\
    \dot{x}_2 &=
        -2k\,\operatorname{sign}(x_2)|x_2|^{3/2}
        +2\sqrt{|x_2|}\,u,
\end{aligned}
\label{eq:motivating_plant}
\end{equation}
where \(x_0\) is position, \(x_1\) is velocity, \(x_2\) is normalized thrust/acceleration, and \(f_e(t)\) is an external force. 
Let \(x_{2d}\) be the thrust reference generated by an impedance, admittance, or force-tracking controller. In any case, exact tracking \(x_2(t)\equiv x_{2d}(t)\) requires, whenever \(x_{2d}(t)\neq0\),
\begin{equation}
    u(t)=k x_{2d}(t)+\frac{\dot{x}_{2d}(t)}{2\sqrt{|x_{2d}(t)|}}.
    \label{eq:tracking_u_motiv}
\end{equation}
Therefore, if \(x_{2d}(t_s)=0\) and \(\dot x_{2d}(t_s)\neq0\), the required motor input is unbounded as \(t\to t_s\). 
This is exactly the \(r(t_s)=2\) case of Theorem~\ref{thm:in_out_feasibility}.

As an illustration, suppose the mass is commanded to follow \(x_r(t)=0.35t\) under the external force \(f_e(t)=0.6\sin(2\pi t+5\pi/2)\). 
An impedance-type thrust reference can be written as
\begin{equation}
x_{2d,\mathrm{imp}}
=
\ddot{x}_r
-\frac{B_I}{M_I}(x_1-\dot{x}_r)
-\frac{K_I}{M_I}(x_0-x_r)
+\left(\frac{1}{M_I}-1\right)f_e,
\label{eq:motivating_imp}
\end{equation}
with \(M_I=0.65\), \(B_I=3.5\), and \(K_I=5.0\) whereas an admittance-type reference first generates a compliant trajectory \(x_a\) through
\begin{equation}
\ddot{x}_a
=
\ddot{x}_r
+
\frac{1}{M_A}\bigl(
f_e
-
B_A(\dot{x}_a-\dot{x}_r)
-
K_A(x_a-x_r)
\bigr),
\label{eq:motivating_adm_model}
\end{equation}
and then commands
\begin{equation}
x_{2d,\mathrm{adm}}
=
\ddot{x}_a
-f_e
-k_v(x_1-\dot{x}_a)
-k_p(x_0-x_a),
\label{eq:motivating_adm}
\end{equation}
with \(M_A=0.60\), \(B_A=3.0\), \(K_A=4.5\), \(k_p=7.0\), and \(k_v=5.5\). These references cross zero linearly, hence \eqref{eq:tracking_u_motiv} predicts the singular motor input shown in Fig.~\ref{fig:admit_exa}. 
The point is not that impedance or admittance control is invalid; rather, thrust references generated during interaction must be made actuator-aware when implemented with bidirectional propellers.

\begin{figure}[t]
    \centering
    \includegraphics[width=0.9\columnwidth]{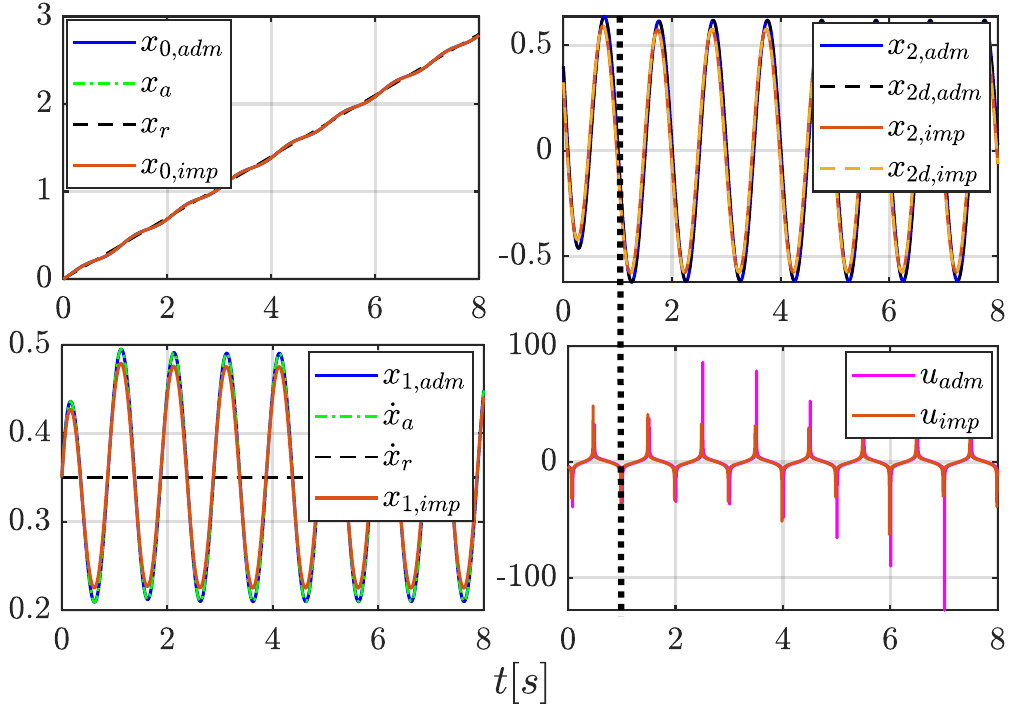}
\caption{Impedance and admittance examples on \eqref{eq:motivating_plant}. 
The ideal motor input becomes singular at instants where the desired thrust crosses zero with nonzero slope during interaction.}
    \label{fig:admit_exa}
\end{figure}

%%%%%%%%%%%%%%%%%%%%%%%%%%%%%%%%%%%%%%%%%
\section{Actuator-Level Experimental Validation}
\label{sec:exp}
%%%%%%%%%%%%%%%%%%%%%%%%%%%%%%%%%%%%%%%%%

\begin{figure*}[t]
    \centering
    \resizebox{\textwidth}{!}{%
    \begin{tikzpicture}[
        line cap=round,
        line join=round,
        >=Latex,
        font=\scriptsize,
        photo/.style={draw=black, line width=0.8pt, fill=gray!8},
        tag/.style={draw=black, rounded corners=2pt, line width=0.6pt, fill=white, align=center, inner sep=3pt},
        block/.style={draw=black, rounded corners=2pt, line width=0.7pt, fill=gray!5, align=center, inner sep=4pt},
        act/.style={draw=black, rounded corners=2pt, line width=0.7pt, fill=orange!8, align=center, inner sep=4pt},
        sens/.style={draw=black, rounded corners=2pt, line width=0.7pt, fill=green!6, align=center, inner sep=4pt},
        ctrl/.style={draw=black, rounded corners=2pt, line width=0.7pt, fill=blue!5, align=center, inner sep=4pt},
        sig/.style={black, line width=0.85pt, -{Latex[length=2.0mm,width=1.4mm]}},
        power/.style={red!70!black, line width=0.9pt, -{Latex[length=2.0mm,width=1.4mm]}},
        meas/.style={green!45!black, line width=0.9pt, -{Latex[length=2.0mm,width=1.4mm]}},
        callout/.style={black!65, line width=0.6pt, -{Latex[length=1.6mm,width=1.1mm]}},
        group/.style={draw=black!35, dashed, rounded corners=3pt, line width=0.6pt},
        lab/.style={font=\scriptsize, align=center},
        photoCallout/.style={
    Magenta!90!black,
    line width=1.25pt,
    -{Latex[length=1.6mm,width=1.7mm]},
    shorten >=1pt,
    shorten <=1pt
},
    ]

    % =====================================================
    % User-adjustable scale for the left setup-photo panel
    % =====================================================
    \def\setupPanelScale{0.99} % 1.00 original, 0.90 smaller, 0.80 much smaller
    \pgfmathsetmacro{\rightPanelX}{7.3*\setupPanelScale + 1.6}

    % Bounding box, automatically adjusted to the right-panel shift
    \path[use as bounding box] (-0.15,0.25) rectangle ({\rightPanelX+8.35},-5.95);

% =====================================================
% Left panel: setup photo with callouts
% =====================================================
\begin{scope}[scale=\setupPanelScale, transform shape]

    % Photo panel
    \node[photo, minimum width=7.3cm, minimum height=4.8cm, anchor=north west] (photo) at (0,0) {};

    % Actual setup photo
    \node[anchor=center] at (photo.center)
    {\includegraphics[width=7.3cm,height=4.8cm,keepaspectratio]{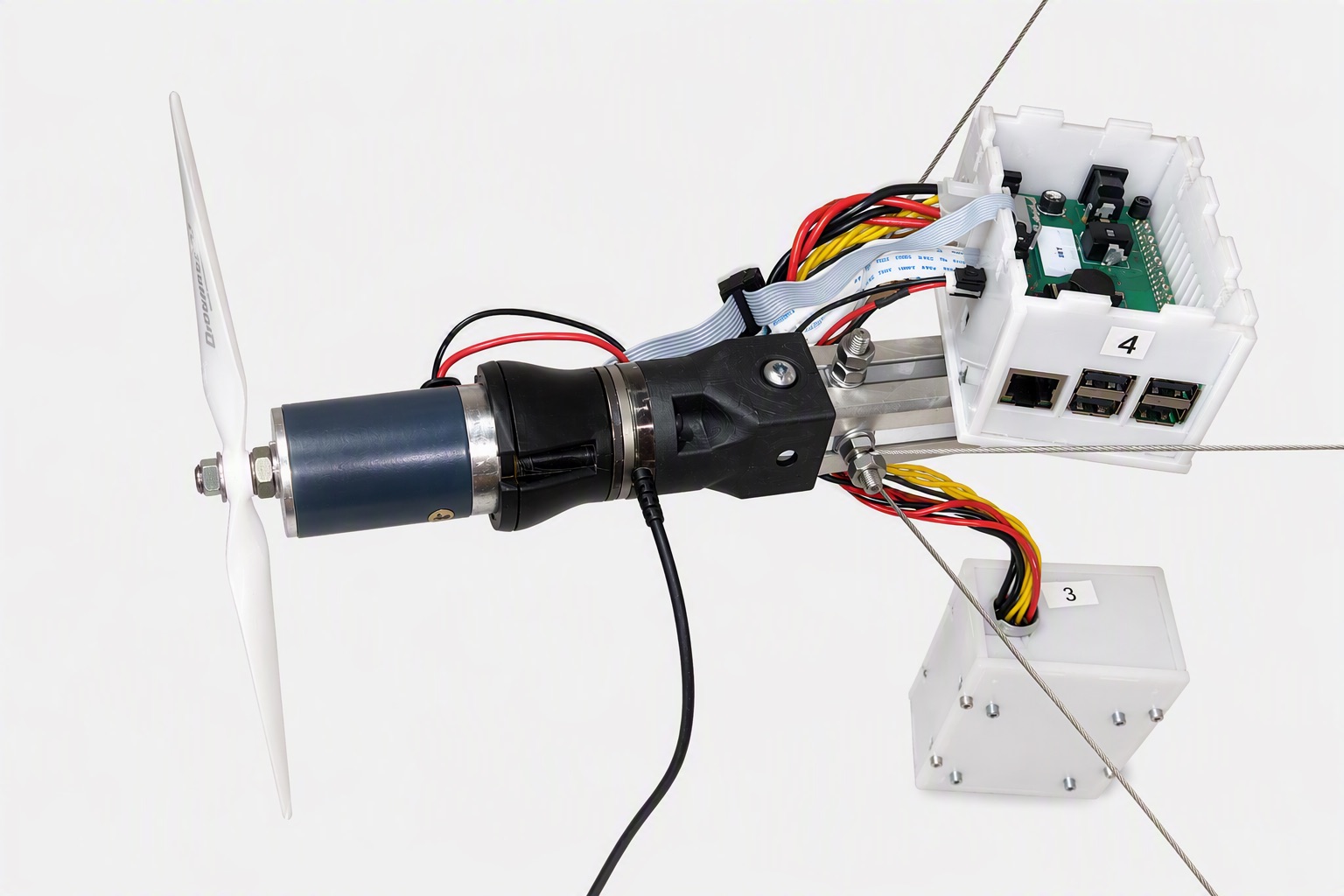}};

    % -----------------------------------------------------
    % Callout tags: three above, three below
    % Coordinates are in the 7.3 cm x 4.8 cm photo frame
    % -----------------------------------------------------

    % Top row labels
    \node[tag, anchor=south] (p_tag)  at (2.3, -0.9) {bidirectional\\propeller};
    \node[tag, anchor=south] (m_tag)  at (4, -0.9) {DC motor};
    \node[tag, anchor=south] (el_tag) at (6.4, -0.9) {electronics\\box};

    % Bottom row labels
    \node[tag, anchor=north] (enc_tag) at (2.5,-3.92) {rotor\\encoder};
    \node[tag, anchor=north] (ft_tag)  at (4.25,-3.92) {force--torque\\sensor / fixture};
    \node[tag, anchor=north] (cur_tag) at (6.25,-4.12) {power box};

% -----------------------------------------------------
% Callout arrows: short, local, nonintersecting
% -----------------------------------------------------

% Top row arrows
\draw[photoCallout] (p_tag.south)  -- (1.4,-1.8);  % propeller blade/hub region
\draw[photoCallout] (m_tag.south)  -- (1.9,-2.4);    % motor cylinder
\draw[photoCallout] (el_tag.south) -- (6.05,-1.45); % electronics box

% Bottom row arrows
\draw[photoCallout] (enc_tag.north) -- (3.05,-2.40); % rear motor / encoder region
\draw[photoCallout] (ft_tag.north)  -- (3.45,-2.40); % black central fixture / sensor mount
\draw[photoCallout] (cur_tag.north) -- (5.60,-3.65); % power box / wiring region

    \node[font=\small\bfseries, anchor=north west] at (0,-5.02)
        {(a) Experimental motor--propeller setup};

\end{scope}

    % =====================================================
    % Right panel: clean signal diagram
    % =====================================================
    \begin{scope}[shift={(\rightPanelX,0)}]

        % Main actuation blocks
        \node[ctrl, minimum width=1.15cm, minimum height=0.65cm] (controller) at (0.0,-0.75) {controller/\\logger};
        \node[act, minimum width=1.05cm, minimum height=0.65cm] (driver) at (2.25,-0.75) {motor\\driver};
        \node[act, minimum width=1.05cm, minimum height=0.65cm] (motor) at (4.70,-0.75) {DC\\motor};
        \node[act, minimum width=1.60cm, minimum height=0.65cm] (propeller) at (6.95,-0.75) {bidirectional\\propeller};

        % Measurement blocks
        \node[sens, minimum width=1.65cm, minimum height=0.65cm] (current) at (2.25,-2.35) {current\\monitor};
        \node[sens, minimum width=1.65cm, minimum height=0.65cm] (encoder) at (4.70,-2.35) {encoder};
        \node[sens, minimum width=1.80cm, minimum height=0.65cm] (force) at (6.95,-2.35) {force--torque\\sensor};
        \node[sens, minimum width=1.25cm, minimum height=0.55cm] (adc) at (0.0,-2.35) {ADC};

        % Actuation chain
        \draw[sig] (controller.east) -- node[lab, above] {command} (driver.west);
        \draw[power] (driver.east) -- node[lab, above] {voltage} (motor.west);
        \draw[sig] (motor.east) -- node[lab, above] {speed} (propeller.west);

        % Measurement arrows, clean downward mapping
        \draw[meas] (driver.south) -- node[lab, left] {current} (current.north);
        \draw[meas] (motor.south) -- node[lab, left] {position} (encoder.north);
        \draw[meas] (propeller.south) -- node[lab, right] {thrust} (force.north);

        % Measurement paths back to controller
        \coordinate (bus) at (0.0,-3.35);

        \draw[meas] (current.west) -- (adc.east);
        \draw[meas] (adc.north) -- (controller.south);

        \draw[meas] (encoder.south) |- (bus) -- (adc.south);
        \draw[meas] (force.south) |- (bus);

        % Group boxes
        \node[group, fit=(driver)(motor)(propeller), inner sep=7pt,
              label={[font=\scriptsize]above:actuation side}] {};
        \node[group, fit=(adc)(current)(encoder)(force), inner sep=7pt,
              label={[font=\scriptsize]below:measurement side}] {};

        \node[font=\small\bfseries, anchor=north west] at (-0.95,-3.85)
            {(b) Signals used in the validation};

    \end{scope}
    \end{tikzpicture}%
    }
\caption{Actuator-level setup and signal flow. 
The DC motor drives the bidirectional propeller, while rotor position, motor current, applied voltage, and axial thrust are recorded around zero-thrust crossings. 
The controller/logger computes the voltage command from rotor-position, rotor-speed, and thrust errors; current and voltage are recorded for validation.}
    \label{fig:experimental_setup}
\end{figure*}

\begin{figure*}[t]
    \centering
        \includegraphics[width=1\linewidth]{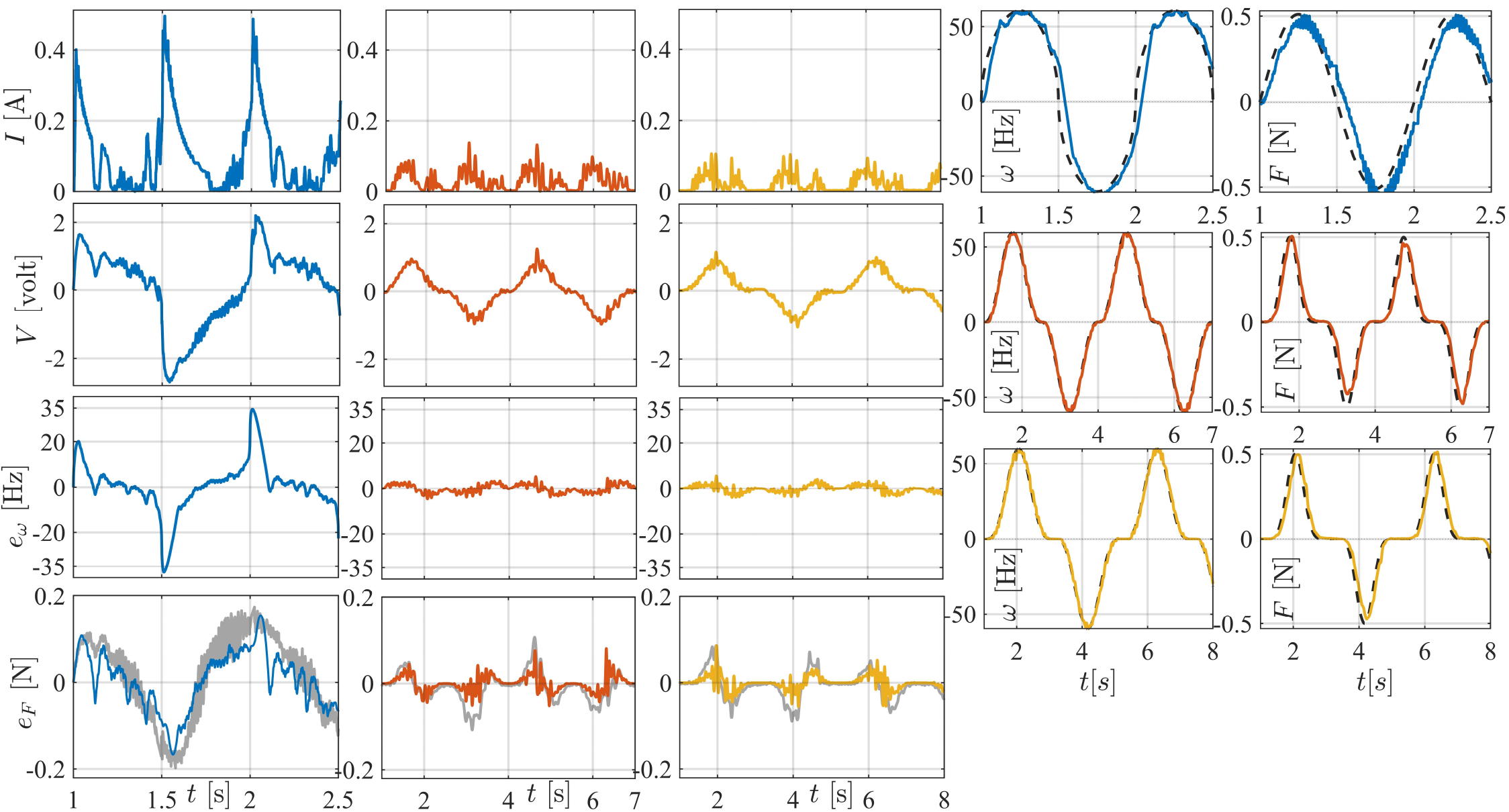}
\caption{Experimental comparison of three thrust-reversal profiles. 
Left block: motor current \(I\), applied voltage \(V\), speed error \(e_\omega\), and thrust-force error \(e_F\) for Case~1 (blue), Case~4 (red), and Case~5 (yellow). 
Right block: speed \(\omega\) and thrust \(F\) tracking, with dashed curves denoting references and gray traces denoting force-sensor measurements where shown. 
The linear reversal in Case~1 produces localized voltage/current peaks and large tracking errors near zero thrust, whereas Cases~4 and~5 flatten the thrust profile at zero and avoid this degradation.}
    \label{fig:dc_exp_cases}
\end{figure*}

We validate the reversal-order predictions on the actuator-level setup shown in Fig.~\ref{fig:experimental_setup}. 
The experiment isolates the motor-propeller mechanism through which thrust reversals demand bounded or unbounded motor commands. 
Following Table~\ref{tab:practical_rules}, Fig.~\ref{fig:dc_exp_cases} reports the experimental results comparing one linear reversal, with \(p=1\), against two higher-order sign-changing reversals, with \(p=5\) and \(p=7\).

The setup uses a \(12\,\mathrm{V}\) Zhengke ZYTD-38S-R brushed DC motor in the \(2000\,\mathrm{r\,min^{-1}}\) configuration, driving an \(8\,\mathrm{g}\), \(0.20\,\mathrm{m}\)-diameter, \(0.025\,\mathrm{m}\)-wide propeller with inertia \(J_p=2.71\times10^{-5}\,\mathrm{kg\,m^2}\). 
The identified parameters are \(B=5.31\times10^{-5}\,\mathrm{N\,m\,s\,rad^{-1}}\), \(c_f=1.4\times10^{-4}\,\mathrm{N\,Hz^{-2}}\), and \(c_\tau=2.2\times10^{-6}\,\mathrm{N\,m\,Hz^{-2}}\). 
Rotor position \(\theta\) is measured by an AS5048B magnetic encoder, rotor speed \(\omega\) is estimated by a linear observer at \(T_s=1\,\mathrm{ms}\), axial thrust is measured by an ATI Mini40 force-torque sensor, and motor current is digitized through an AD7091R ADC.

The desired rotor position is obtained from \(\theta_d(t)=\int_0^t\omega_d(\tau)d\tau\), with \(\omega_d=\operatorname{sgn}(x_{2d})\sqrt{|x_{2d}|/c_f}\). 
The voltage command is
\(
u=K_pe_\theta+K_i\int e_\theta dt+K_de_\omega+K_{pF}e_F+K_{iF}\int e_Fdt,
\)
where \(e_\theta=\theta_d-\theta\), \(e_\omega=\omega_d-\hat\omega\), and \(e_F=F_d-F\). 
The gains are \(K_p=190\), \(K_i=5\), \(K_d=7\), \(K_{pF}=500\), and \(K_{iF}=5\). 
A \(3\,\mathrm{ms}\) low-pass filter is applied to \(e_\omega\), a \(0.02\,\mathrm{s}\) filter to \(e_F\), and the remaining measurements are filtered at \(15\,\mathrm{Hz}\). 
Anti-windup, one-count quantization, and a \(\pm2048\)-count command range are used, giving \(5.86\,\mathrm{mV}\) per count and a voltage limit of \(\pm12\,\mathrm{V}\).

The three experiments use \(x_{2d}=A\sin(s_1)\), \(x_{2d}=A\sin^5(s_2)\), and \(x_{2d}=A\sin^7(s_3)\), corresponding to Cases~1, 4, and~5. 
Different periods are used to compare the cases without feedforward compensation. 
Since Case~1 admits no bounded ideal voltage feedforward at reversal, the periods of Cases~4 and~5 are selected to keep the nonsingular portions of the rotor-speed profiles comparable to Case~1. 
Thus, the comparison emphasizes zero-crossing order rather than maneuver speed alone.

As shown in Fig.~\ref{fig:dc_exp_cases}, Case~1 has \(p=1\) and \(r=2\), so exact reproduction would require an unbounded ideal motor input. 
Cases~4 and~5 have \(p=5\) and \(p=7\), respectively, and therefore fall in the voltage-continuous and smoother regimes of Table~\ref{tab:practical_rules}. 
This distinction appears in the left block through sharp current/voltage peaks and large error excursions for Case~1, and in the right block through the speed and thrust degradation localized at its zero crossings.

The global tracking errors support this observation. 
The speed mean-absolute errors are approximately \(7.27\,\mathrm{Hz}\), \(1.65\,\mathrm{Hz}\), and \(1.26\,\mathrm{Hz}\), while the speed RMS errors are \(11.46\,\mathrm{Hz}\), \(1.95\,\mathrm{Hz}\), and \(1.64\,\mathrm{Hz}\), for Cases~1, 4, and~5, respectively. 
The corresponding thrust RMS errors are \(0.100\,\mathrm{N}\), \(0.036\,\mathrm{N}\), and \(0.031\,\mathrm{N}\). 
The measured thrust ranges are approximately \([-0.593,0.588]\,\mathrm{N}\), \([-0.481,0.506]\,\mathrm{N}\), and \([-0.488,0.513]\,\mathrm{N}\), showing mild thrust asymmetry. 
Overall, Cases~4 and~5 track similarly well, whereas Case~1 has about three times the thrust RMS error and much larger reversal-localized errors.

The local behavior around zero thrust reveals the mechanism most directly. 
Within approximately \(50\,\mathrm{ms}\) of each reversal, Case~1 reaches a maximum local speed error of about \(37.6\,\mathrm{Hz}\), together with localized voltage and current peaks of about \(2.70\,\mathrm{V}\) and \(0.493\,\mathrm{A}\). 
These are finite sampled and saturated peaks; the singular quantity is the ideal motor command predicted by the inverse model. 
By contrast, the corresponding local speed errors in Cases~4 and~5 are about \(1.28\,\mathrm{Hz}\) and \(0.10\,\mathrm{Hz}\), with reversal currents below approximately \(4.0\,\mathrm{mA}\) and \(4.1\,\mathrm{mA}\). 
Their larger global current peaks, approximately \(0.149\,\mathrm{A}\) and \(0.141\,\mathrm{A}\), occur away from zero thrust and are associated with ordinary rotor acceleration and deceleration.

This contrast matches the local scaling behind Table~\ref{tab:practical_rules}. 
For the linear reversal, \(\omega_d\sim\operatorname{sgn}(t-t_s)|t-t_s|^{1/2}\), so \(\dot\omega_d\) is singular at the crossing. 
For Cases~4 and~5, \(\omega_d\) scales as \(\operatorname{sgn}(t-t_s)|t-t_s|^{5/2}\) and \(\operatorname{sgn}(t-t_s)|t-t_s|^{7/2}\), respectively, giving finite and smoother motor references at zero thrust. 
The experiment therefore confirms the main qualitative prediction: reversal-localized tracking degradation and current/voltage peaks appear for generic linear crossings and disappear when the desired thrust is sufficiently flattened at zero.

%%%%%%%%%%%%%%%%%%%%%%%%%%%%%
\section{Conclusion}
\label{conc}
%%%%%%%%%%%%%%%%%%%%%%%%%%%%%

We established exact reproducibility limits for bidirectional propellers under bounded motor inputs. 
The key obstruction is that bidirectional thrust is generated through a signed-quadratic speed-to-thrust map: every thrust reversal must pass through zero rotor speed, where the input authority over the thrust derivative vanishes. 
By deriving a thrust-coordinate normal form, we showed that generic thrust reversals, namely zero crossings with nonzero thrust derivative, require unbounded motor inputs and are therefore not exactly reproducible.

The resulting conditions also provide practical design rules. 
A desired thrust reversal must be sufficiently flat at zero, with the required flatness depending on whether bounded, continuous, or smoother torque/current/voltage commands are needed. 
In particular, linear thrust crossings should be avoided, cubic crossings are the minimum for continuous torque/current, and quintic or higher-order sign-changing crossings provide a practical choice when voltage continuity is required.

We translated these limits into motor-current and motor-voltage requirements for a DC motor driving a bidirectional propeller. 
Experiments on a motor-propeller setup confirmed the predicted reversal-order effect: linear thrust crossings produced localized current/voltage peaks and tracking degradation near zero thrust, whereas higher-order crossings avoided this localized degradation. 
Thus, the observed behavior is a manifestation of the actuator-level reproducibility limit.

These results show that bidirectional propellers should not be treated as ideal signed thrust sources when force, acceleration, or interaction-control references change sign. 
Impedance, admittance, and force-tracking controllers for aerial physical interaction should therefore generate actuator-aware thrust references, for example by flattening, retiming, or otherwise regularizing reversals. 

Future work will integrate these reproducibility constraints into reference generation, control allocation, and whole-body interaction control for fully actuated aerial robots.

\bibliographystyle{IEEEtran}
\bibliography{Refs}
\end{document}